\documentclass[letterpaper]{article} 
\usepackage{aaai2027}  
\usepackage[hyphens]{url}  
\usepackage{graphicx} 
\usepackage{natbib}  
\usepackage{caption} 
\usepackage{algorithm}
\usepackage{algorithmic}
\usepackage{amsmath}
\usepackage{amsthm}
\usepackage{amssymb}
\usepackage{bm}
\usepackage{booktabs}
\usepackage{array}
\usepackage{amsthm}

\newtheorem{corollary}{Corollary}
\newtheorem{lemma}{Lemma}
\newtheorem{theorem}{Theorem}
\newtheorem{assumption}{Assumption}

\usepackage{newfloat}
\usepackage{listings}
\DeclareCaptionStyle{ruled}{labelfont=normalfont,labelsep=colon,strut=off} 
\floatstyle{ruled}
\newfloat{listing}{tb}{lst}{}
\floatname{listing}{Listing}

\usepackage{booktabs}

\affiliations{
    \textsuperscript{\rm 1}Association for the Advancement of Artificial Intelligence\\

    1101 Pennsylvania Ave, NW Suite 300\\
    Washington, DC 20004 USA\\
    proceedings-questions@aaai.org
}

\title{DIB-OD: Preserving the Invariant Core for Robust Heterogeneous Graph Adaptation via Decoupled Information Bottleneck and Online Distillation}
\author {
    Yang Yan\textsuperscript{\rm 1},
    Yunxuan Li\textsuperscript{\rm 1},
    Qiuyan Wang\textsuperscript{\rm 2},
    Tianjin Huang\textsuperscript{\rm 3},
    Qiudong Yu\textsuperscript{\rm 1}\corresponding
}
\affiliations {
    \textsuperscript{\rm 1}School of Information Technology and Engineering, Tianjin University of Technology and Education, Tianjin, China \\
    \textsuperscript{\rm 2}School of Computer Science and Technology, Tiangong University, Tianjin, China \\
    \textsuperscript{\rm 3}Department of Computer Science at University of Exeter \\
    yanyangucas@126.com, qiudongyu2027@126.com, wangyan198801@163.com, cxy19980916@gmail.com,  t.huang2@exeter.ac.uk
}

\begin{document}

\maketitle

\begin{abstract}
Graph pre-training can facilitate knowledge transfer across graph
datasets, but severe structural and feature shifts may cause
negative transfer and adaptation-induced overwriting of reusable
knowledge.
We propose \textbf{DIB-OD}, a heterogeneous graph adaptation
framework that combines a \textbf{D}ecoupled
\textbf{I}nformation \textbf{B}ottleneck with
\textbf{O}nline \textbf{D}istillation.
A multi-view teacher first learns a compressed, task-relevant
representation, which is distilled into a transferable core branch
and a complementary residual branch.
Mutual-information objectives and HSIC-based dependence
regularization discourage information overlap between the branches,
while a confidence-aware semantic regularizer and a frozen teacher
anchor reliable pretrained information during target-domain
adaptation.
Experiments on seven graph-classification datasets spanning
chemical, biological, and social-network domains show consistent
improvements over representative baselines, particularly under
challenging cross-type transfer.
\end{abstract}


\section{Introduction}

Graph Neural Networks  have revolutionized representation learning across diverse fields, from molecular property prediction to social network analysis. 
Graph pre-training learns transferable representations from
unlabeled graphs before downstream adaptation
\cite{hu2020strategies,sun2020infograph,hou2022graphmae}.
 This pursuit of cross-domain generalization is also a fundamental challenge in many AI fields.
 Despite its success, a critical generalization wall persists: most existing GNN pre-training and domain adaptation frameworks are designed for intra-domain shifts—addressing distribution variations within the same category of data (e.g., regional shifts in citation networks or sensor noise in protein graphs).

Recently, an emerging trend has seen the development of universal graph models that aim to aggregate knowledge across disparate data sources~\cite{zhao2025fully,he2025unigraph,2025Homophily}. 
While these pioneer efforts have demonstrated the potential of large-scale graph learning, a critical representation bottleneck remains largely unaddressed across the field.
Existing pre-training strategies, including those emerging universal frameworks, often treat latent features as monolithic entities. 
This results in a fundamental information entanglement: domain-specific structural styles are deeply entangled with the task-relevant topological laws. 
In heterogeneous transfer scenarios, the noise's representation from a source domain  often acts as a spurious correlation that misleads the model when it encounters a target domain with a completely different connectivity logic, leading to negative transfer.

Furthermore, heterogeneous target-domain adaptation introduces a fundamental stability–plasticity tension for pretrained GNNs.
 Although the model should remain sufficiently plastic to capture target-specific patterns, unrestricted fine-tuning may corrupt the transferable structural principles acquired during pre-training. 
 Without an explicit mechanism to isolate and protect the invariant core, target-specific signals can overwrite generalizable knowledge, resulting in adaptation-induced invariant-core corruption.

\paragraph{Problem Setting.}
We study heterogeneous supervised graph transfer, in which a model
is first pretrained on labeled source-domain graphs and subsequently
adapted using labeled graphs from the target training split. The
held-out target test split is not used for training, model selection,
or construction of the semantic regularizer. Source and target
datasets may differ substantially in graph structure, feature
statistics, and application-specific prediction semantics.
Accordingly, DIB-OD transfers graph representations rather than
assuming direct alignment of source and target class semantics.
This setting differs from conventional unsupervised closed-set graph
domain adaptation, where target labels are unavailable and the task
semantics are typically shared across domains.

To address these limitations, we propose DIB-OD, a heterogeneous
graph transfer framework based on Decoupled Information Bottleneck
and Online Distillation. Our approach shifts the focus from
representation aggregation to knowledge purification. Built upon a
multi-view teacher--student architecture, DIB-OD allocates learned
information between two branches: (1) a task-relevant core branch
that is encouraged to retain transferable predictive information,
and (2) a complementary residual branch that preserves
instance-level information not allocated to the core. The information
bottleneck controls label-information allocation, while the
Hilbert--Schmidt Independence Criterion provides a tractable
dependence penalty that discourages the two branches from encoding
identical content. The terms ``core'' and ``residual'' describe their
intended optimization roles rather than claiming exact statistical
independence or identifiable recovery of latent causal factors.


To resolve this stability–plasticity tension, we introduce a self-adaptive semantic regularizer that selectively preserves reliable invariant semantics while allowing uncertain and domain-specific components to adapt.
Unlike fixed constraints, this regularizer dynamically modulates the influence of target labels based on predictive confidence, thereby reducing the risk that reliable core semantics are distorted from being overwritten by noisy target signals.
\begin{itemize}
\item  \textbf{A Purification Paradigm for Heterogeneous Transfer:} We identify the ``information entanglement" issue in universal graph learning and propose a novel decoupling framework that bridges fundamentally different graph categories.
\item  \textbf{Decoupled Information Bottleneck Distillation:} We introduce an optimization scheme that combines IB-guided information allocation, online contrastive distillation, and HSIC
dependence regularization to construct complementary core and
residual representations.
\item  \textbf{Self-adaptive Knowledge Preservation:} We design a confidence-aware semantic regularizer together with a frozen-teacher strategy to selectively anchor reliable pretrained semantics while allowing domain-specific components to adapt, thereby mitigating target-induced corruption of the invariant core.

\item \textbf{Extensive Benchmarking:} Extensive experiments on several cross-domain benchmarks demonstrate that DIB-OD significantly outperforms state-of-the-art methods, especially in challenging inter-type domain transfers  (e.g., Chem $\rightarrow$ Bio), showcasing superior generalization.
\end{itemize}

\section{Related Works}

\textbf{Graph transfer and OOD generalization.}
Graph domain adaptation transfers knowledge across graphs with
different structural and feature distributions.
Representative methods include DANE~\cite{zhang2019dane},
UDAGCN~\cite{wu2020unsupervised},
ASN~\cite{zhang2021adversarial},
GRADE~\cite{wu2023non},
JHGDA~\cite{shi2023improving}, and
Pair-Align~\cite{liu2024pairwise}.
Most of these methods focus on related source and target graph
collections
\cite{you2023graph,2025Homophily}.
Graph OOD and universal-model studies instead seek representations
that remain useful across environments or datasets
\cite{li2022learning,chen2023does,gui2022good,
mo2024graph,chen2025cross,zhao2025fully,he2025unigraph}.
DIB-OD focuses on internal knowledge decomposition during
heterogeneous transfer rather than treating the transferred
representation as a monolithic feature vector.

\textbf{Information-theoretic decomposition and distillation.}
Information bottleneck, contrastive learning, and adversarial
objectives have been used to separate shared and view-specific
factors
\cite{yin2024mvaibnet,fan2023generalizing,wu2020graph,
li2021disentangled,du2026gcib,wang2024disentangled}.
Knowledge distillation typically transfers teacher predictions or
relations to a student
\cite{park2019relational,chen2021cross,
joshi2022representation}.
In contrast, DIB-OD jointly uses label-information allocation,
branch-dependence regularization, and online contrastive distillation
to construct a transferable core branch and to preserve it during
target-domain adaptation.

\section{Methodology}

\subsection{Framework Overview}

As Figure~\ref{fig:SystemArchitecture} illustrated, our DIB-OD framework consists of three main components: (1) a multi-view encoder that processes various augmented graph views, (2) a teacher model that learns a fused representation via the Information Bottleneck (IB) principle, and (3) student models that disentangle the representation into invariant core and redundant subspaces through online contrastive distillation.


Our analysis follows the learning pipeline of DIB-OD.
Lemma~\ref{lm:view_equivalence}~and~\ref{lm:view_decomposition}~first justify the view-conditioned teacher objective
and its decomposition over individual augmented views.
Theorem~1 then connects student-side core--residual decoupling to
cross-domain generalization.
Finally, Theorem~2 provides a reliability interpretation for the
class-wise confidence threshold used by the semantic regularizer
during target-domain adaptation.

\begin{figure*}[!ht]
\centering
\includegraphics[width=\textwidth]{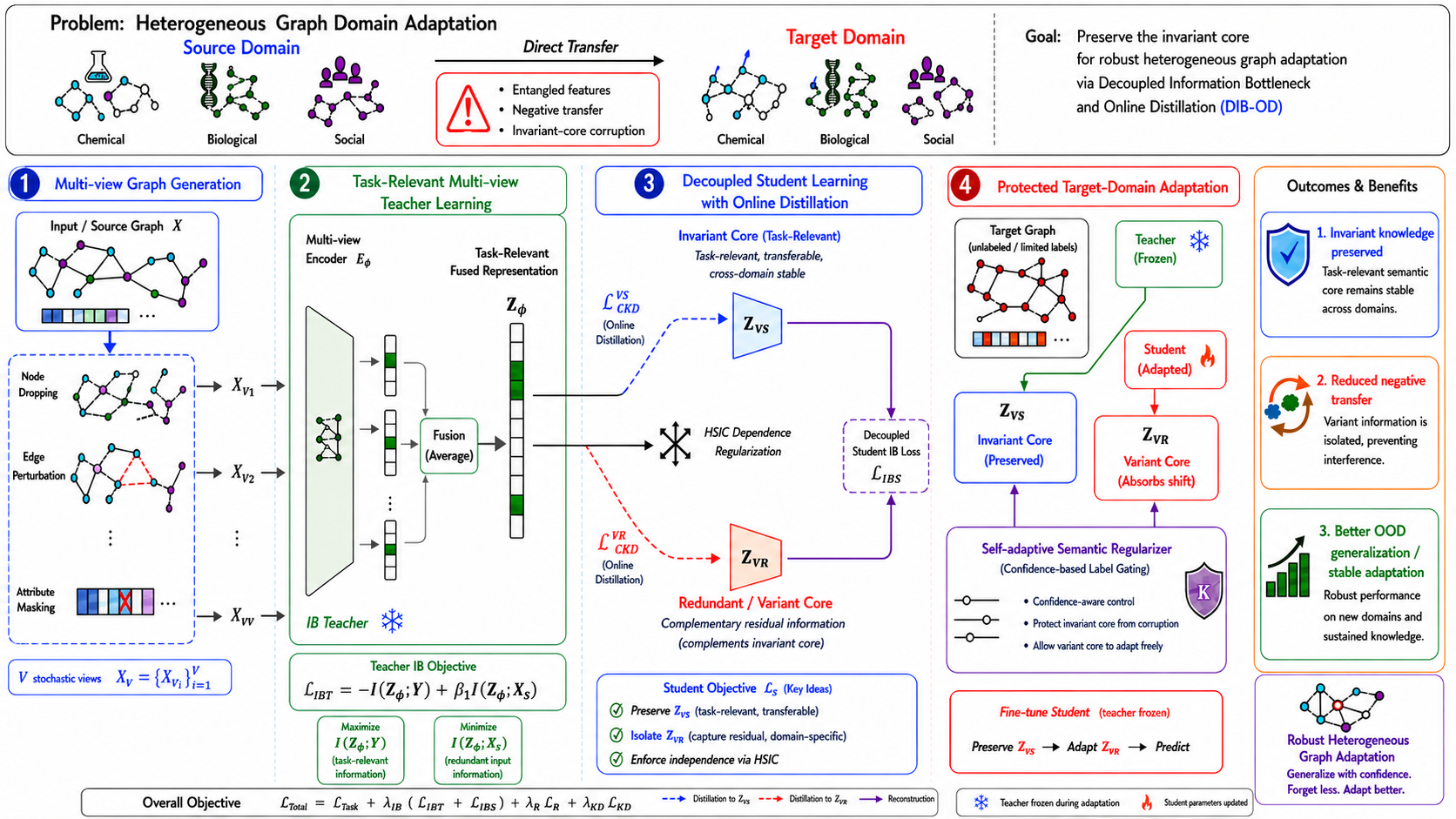}
\caption{Overview of DIB-OD}
\label{fig:SystemArchitecture}
\end{figure*}

\subsection{Multi-view Graph Disentangled Auto-Encoder}
Given an input graph, we construct $V$ stochastic views
$\mathbf{X}_{\Phi}=\{\mathbf{X}^{(1)},\ldots,\mathbf{X}^{(V)}\}$
using node dropping, edge perturbation, and attribute masking.
The corresponding view encoders produce a fused teacher
representation
\begin{align} \label{eq:fused_teacher}
\mathbf{Z}_{\Phi}
=
\sum_{i=1}^{V}\alpha_i E_i(\mathbf{X}^{(i)}),
\qquad
\alpha_i\geq0,\quad
\sum_{i=1}^{V}\alpha_i=1.
\end{align}

\subsection{Task-Relevant Multi-view Teacher Learning via Information Bottleneck}
The teacher regulates the information flow from the augmented inputs
$\mathbf{X}_{\Phi}$ to the fused representation
$\mathbf{Z}_{\Phi}$.
Let $\mathbf{Y}$ denote the labels available in the source-domain
pre-training stage.
The teacher objective is
\begin{align}
\mathcal{L}_{\mathrm{IBT}}
=
-I(\mathbf{Z}_{\Phi};\mathbf{Y})
+
\beta_T I(\mathbf{Z}_{\Phi};\mathbf{X}_{\Phi}),
\label{eq:L_IBT}
\end{align}
where the first term retains task-predictive information and the
second compresses input-specific information.

The augmentation index $\Phi$ and the domain variable $D$ play
different roles.
Specifically, $\Phi$ identifies stochastic views generated within a
dataset, whereas $D$ identifies different source or target datasets.
Consequently, augmentation stability alone does not imply domain
invariance.
DIB-OD addresses heterogeneous transfer through the subsequent
core--residual decomposition and the protected adaptation mechanism,
rather than assuming that view invariance is sufficient for
cross-domain invariance.

\begin{lemma}[View-Conditioned Predictive Equivalence]
\label{lm:view_equivalence}
Let $\Phi$ denote the augmented-view index and let
$\widetilde{\mathbf{Z}}=\mathbf{Z}^{(\Phi)}$ be the representation
extracted from the sampled view.
If
\begin{align}
I(\mathbf{Y};\Phi)=0,
\qquad
I(\widetilde{\mathbf{Z}};\Phi\mid\mathbf{Y})=0,
\label{eq:view_conditions}
\end{align}
then
\begin{align}
I(\widetilde{\mathbf{Z}};\mathbf{Y})
=
I(\widetilde{\mathbf{Z}};\mathbf{Y}\mid\Phi).
\label{eq:view_equivalence}
\end{align}
\end{lemma}

\begin{proof}
The two conditions imply marginal independence between
$\widetilde{\mathbf{Z}}$ and $\Phi$. Indeed,
\begin{align}
p(\widetilde{\mathbf{z}},\phi)
&=
\sum_y
p(\widetilde{\mathbf{z}}\mid y,\phi)p(y,\phi)
\nonumber\\
&=
\sum_y
p(\widetilde{\mathbf{z}}\mid y)p(y)p(\phi)
=
p(\widetilde{\mathbf{z}})p(\phi),
\end{align}
and hence
$I(\widetilde{\mathbf{Z}};\Phi)=0$.
Using the mutual-information identity
\begin{align}
&I(\widetilde{\mathbf{Z}};\mathbf{Y}\mid\Phi)
-
I(\widetilde{\mathbf{Z}};\mathbf{Y})
\nonumber\\
&\qquad =
I(\widetilde{\mathbf{Z}};\Phi\mid\mathbf{Y})
-
I(\widetilde{\mathbf{Z}};\Phi),
\end{align}
both terms on the right-hand side are zero, which proves
Equation~\eqref{eq:view_equivalence}.
\end{proof}
Lemma~\ref{lm:view_equivalence} shows that, under label-preserving
augmentations and conditional view invariance, optimizing predictive
information can be equivalently expressed through a
view-conditioned objective.
\begin{lemma}[View-Wise Decomposition of Predictive Information]
\label{lm:view_decomposition}
Under the notation of Lemma~\ref{lm:view_equivalence}, the
view-conditioned predictive information decomposes as
\begin{align}
I(\widetilde{\mathbf{Z}};\mathbf{Y}\mid\Phi)
=
\sum_{i=1}^{V}
p_i
I\!\left(
\mathbf{Z}^{(i)};
\mathbf{Y}
\mid
\Phi=\phi_i
\right),
\label{eq:view_decomposition}
\end{align}
where $p_i=P(\Phi=\phi_i)$.
For uniformly sampled views, $p_i=1/V$.
\end{lemma}

\begin{proof}
By the definition of conditional mutual information,
\begin{align}
I(\widetilde{\mathbf{Z}};\mathbf{Y}\mid\Phi)
&=
\mathbb{E}_{\Phi}
\left[
I(\widetilde{\mathbf{Z}};\mathbf{Y}\mid\Phi=\phi)
\right]
\nonumber\\
&=
\sum_{i=1}^{V}
p_i
I(\widetilde{\mathbf{Z}};\mathbf{Y}\mid\Phi=\phi_i).
\end{align}
Conditioned on $\Phi=\phi_i$,
$\widetilde{\mathbf{Z}}=\mathbf{Z}^{(i)}$, which yields
Equation~\eqref{eq:view_decomposition}.
\end{proof}
%


Lemma~\ref{lm:view_decomposition} characterizes the
predictive contribution of individual augmented views through the
sampling probabilities $p_i=P(\Phi=\phi_i)$.
These probabilities describe how the conditional objective is
averaged over stochastic views and are distinct from the learnable
fusion coefficients $\alpha_i$ in
Equation~\eqref{eq:fused_teacher}.
The latter are model parameters that allow the teacher to emphasize
views carrying stronger task-relevant information.

\begin{corollary}[Multi-view Predictive Objective]
\label{cor:multi_view_objective}
Under the conditions of Lemma~\ref{lm:view_equivalence},
\begin{align}
I(\widetilde{\mathbf{Z}};\mathbf{Y})
=
\sum_{i=1}^{V}
p_i
I\!\left(
\mathbf{Z}^{(i)};
\mathbf{Y}
\mid
\Phi=\phi_i
\right).
\label{eq:multi_view_objective}
\end{align}
\end{corollary}

\begin{proof}
The result follows directly by combining
Lemmas~\ref{lm:view_equivalence}
and~\ref{lm:view_decomposition}.
\end{proof}
Corollary~\ref{cor:multi_view_objective} provides an
objective-level justification for the fusion in
Equation~\eqref{eq:fused_teacher}.
The sampling probabilities $p_i$ determine the theoretical
view-wise expectation, whereas the learned coefficients $\alpha_i$
control representation aggregation; in general,
$\alpha_i\neq p_i$.



\subsection{Core--Residual Student Learning with Online Distillation}

The fused teacher representation is distilled into two student
branches.
We refer to $\mathbf{z}_{vs}$ as the \emph{invariant-core branch}
because it is optimized to retain label-predictive information, and
to $\mathbf{z}_{vr}$ as the \emph{residual branch} because its
label-predictive information is suppressed while complementary
instance-level teacher information is retained.
These names describe their optimization roles; they do not assume
that the two latent factors are causally identifiable.

The decoupled student objective is
\begin{align}
\mathcal{L}_{\mathrm{IBS}}
=
&-\kappa I(\mathbf{z}_{vs};\mathbf{Y})
+\beta_y I(\mathbf{z}_{vr};\mathbf{Y}) \nonumber\\
&+\beta_{vs}I(\mathbf{z}_{vs};\mathbf{z}_{vr})
-I(\mathbf{z}_{vr};\mathbf{Z}_{\Phi}).
\label{eq:L_IBS}
\end{align}
We use the contrastive losses
$\mathcal{L}_{\mathrm{CKD}}^{vs}$ and
$\mathcal{L}_{\mathrm{CKD}}^{vr}$ to align the student branches with
instance-level teacher representations; their detailed formulations
are provided in the supplementary material:
\begin{align}\label{eq:L_CKD}
\mathcal{L}_{\text{CKD}} = \mathcal{L}_{\text{CKD}}^{vs} + \mathcal{L}_{\text{CKD}}^{vr}\end{align}
Additionally, we enforce dependence regularization between $\mathbf{z}_{vs}$ and $\mathbf{z}_{vr}$ via the Hilbert-Schmidt Independence Criterion:
\begin{align}\label{eq:L_Orth}
 \mathcal{L}_{\text{Orth}} = \operatorname{HSIC}\left(\mathbf{Z}_{vs}^{S}, \mathbf{Z}_{vr}^{S}\right) = \frac{1}{(N-1)^{2}} \operatorname{tr}\left(\mathbf{K}_{vs}^{S}\mathbf{H}\mathbf{K}_{vr}^{S}\mathbf{H}\right)
\end{align}
where $\mathbf{K}_{vs}^{S}$ and $\mathbf{K}_{vr}^{S}$ are kernel matrices computed from the respective representations, and $\mathbf{H}$ is a centering matrix. 
The total distillation loss is $\mathcal{L}_{\mathrm{KD}} = \mathcal{L}_{\mathrm{CKD}} + \lambda_{\mathrm{Orth}} \mathcal{L}_{\mathrm{Orth}}$.

\begin{assumption}[Residual Sufficiency]
\label{as:dis1}
Let $D\in\{\mathcal{S},\mathcal{T}\}$ denote the domain variable.
For the theoretical analysis, we assume that after conditioning on
the residual representation $\mathbf{z}_{vr}$, the core
representation contains no additional domain information:
\begin{align}
I(\mathbf{z}_{vs};D\mid\mathbf{z}_{vr})=0.
\label{eq:I_D_S}
\end{align}
\end{assumption}

Assumption~\ref{as:dis1} is an idealized condition where domain-specific information is absorbed by the residual branch.
Our objectives encourage, but do not guarantee, this condition.

\begin{theorem}[Information-Decoupling Transfer Bound]
\label{th:gbd}
Let $\epsilon_{\mathcal{S}}(h)$ and
$\epsilon_{\mathcal{T}}(h)$ denote the source and target risks of
a hypothesis $h$ defined on $\mathbf{z}_{vs}$.
Under Assumption~\ref{as:dis1} and uniform domain priors,
\begin{align}
\epsilon_{\mathcal{T}}(h)
\leq
\epsilon_{\mathcal{S}}(h)
+
\sqrt{2I(\mathbf{z}_{vs};\mathbf{z}_{vr})}
+
\lambda,
\label{eq:mi_transfer_bound}
\end{align}
where $\lambda$ is the risk of the ideal joint hypothesis.
\end{theorem}

\noindent\textit{Proof sketch.}
The standard domain-adaptation bound gives
$\epsilon_{\mathcal{T}}
\leq
\epsilon_{\mathcal{S}}
+
\frac{1}{2}d_{\mathcal{H}\Delta\mathcal{H}}
+\lambda$.
The divergence term is bounded by total variation, which under a
uniform domain prior is at most
$\sqrt{2I(\mathbf{z}_{vs};D)}$.
The mutual-information chain rule and
Assumption~\ref{as:dis1} yield
$I(\mathbf{z}_{vs};D)
\leq
I(\mathbf{z}_{vs};\mathbf{z}_{vr})$.
The complete derivation is provided in the supplementary material.
For the detailed proof, please refer to Appendix A.

The theorem motivates minimizing the branch-dependence term in
Equation~\eqref{eq:L_IBS}.
HSIC provides an additional tractable empirical dependence penalty,
but the theorem does not identify HSIC with mutual information.

\subsection{Self-adaptive Semantic Regularizer}
%
%
In DIB-OD, the key to preserving the pretrained invariant core lies in the adaptive formulation of the student Information Bottleneck objective.
The Self-adaptive Semantic Regularizer $\kappa$ (SSR for short defined in Equation~\eqref{eq:L_IBS}) , which dictates the importance of the invariant representation $\mathbf{z}_{vs}$, is dynamically determined. 
The intuition is that knowledge corresponding to semantic classes that the pre-trained model can predict with high stability and confidence is likely part of the generalizable, domain-invariant core. 
This is precisely the knowledge we must preserve. In contrast, knowledge for unstable or uncertain classes may be domain-specific or poorly learned, thus requiring more adaptation.

To quantify this invariant core, we devise a three-step process to compute a dynamic class-aware $\kappa$.
Let $\mathring{y}$ be the prediction of the pre-trained teacher model and $y^{\ast}$ be the ground-truth label.

1. \textbf{Quantifying Pre-trained Confidence}. We first establish a confidence baseline $t_b$ for each class $b$ in Equation~\eqref{eq:t_b}.
This threshold represents the teacher model's average predictive confidence over all target-training graphs predicted as class $b$, serving as a measure of its self-assessed certainty from pre-training.
\begin{align}\label{eq:t_b}
t_b = \frac{1}{ | \mathbf{X}_{ \mathring{y}=b} |} \sum_{\mathbf{x} \in \mathbf{X}_{ \mathring{y}=b}} \hat{p}(\mathring{y}=b; \mathbf{x})
\end{align}


Equation~\eqref{eq:t_b} is the empirical counterpart of the
following population-level class reliability.

\begin{theorem}[Reliability Interpretation of the Class Threshold]
\label{tm:class_threshold}
Let
\begin{align}
\hat{y}(\mathbf{x})
=
\arg\max_{r\in\{1,\ldots,m\}}
\hat{p}_{r}(\mathbf{x}).
\notag
\end{align}
Suppose that the frozen teacher is class-conditionally calibrated,
i.e.,
\begin{align}
\mathbb{E}
\left[
\mathbb{I}\{Y=c\}
\mid
\hat{p}_{c}(\mathbf{X}),
\hat{y}(\mathbf{X})=c
\right]
=
\hat{p}_{c}(\mathbf{X}).
\label{eq:class_calibration} 
\end{align}
Then the population class threshold
\begin{align}
t_c
=
\mathbb{E}
\left[
\hat{p}_{c}(\mathbf{X})
\mid
\hat{y}(\mathbf{X})=c
\right]
\notag
\end{align}
equals the precision of predictions assigned to class $c$:
\begin{align}
t_c
=
P
\left(
Y=c
\mid
\hat{y}(\mathbf{X})=c
\right).
\label{eq:threshold_precision}
\end{align}
\end{theorem}

\begin{proof}
Let
\[
\hat{y}(\mathbf{X})
=
\arg\max_{r\in\{1,\ldots,m\}}
\hat{p}_{r}(\mathbf{X}).
\]
By definition, the population-level confidence threshold for class
$c$ is
\begin{align}
t_c
=
\mathbb{E}
\left[
\hat{p}_{c}(\mathbf{X})
\mid
\hat{y}(\mathbf{X})=c
\right].  
\label{eq:population_threshold}  
\end{align}

Under the class-conditional calibration assumption,
\begin{align} 
\mathbb{E}
\left[
\mathbb{I}\{Y=c\}
\mid
\hat{p}_{c}(\mathbf{X}),
\hat{y}(\mathbf{X})=c
\right]
=
\hat{p}_{c}(\mathbf{X}),
\label{eq:supp_class_calibration}
\end{align}
where $\mathbb{I}\{\cdot\}$ denotes the indicator function.

Substituting Equation~\eqref{eq:supp_class_calibration} into
Equation~\eqref{eq:population_threshold} gives
\begin{align}
t_c  
&=
\mathbb{E}
\Big[
\mathbb{E}
\big[
\mathbb{I}\{Y=c\}
\mid
\hat{p}_{c}(\mathbf{X}),
\hat{y}(\mathbf{X})=c
\big]
\mid
\hat{y}(\mathbf{X})=c
\Big].  \nonumber
\end{align}

By the law of iterated expectation,
\begin{align}
t_c
&=
\mathbb{E}
\left[
\mathbb{I}\{Y=c\}
\mid
\hat{y}(\mathbf{X})=c
\right]
\nonumber\\
&=
P
\left(
Y=c
\mid
\hat{y}(\mathbf{X})=c
\right).  \nonumber
\end{align}

Therefore, under class-conditional calibration, the mean confidence
of instances predicted as class $c$ equals the precision of the
teacher predictions assigned to class $c$.
\end{proof}

2. \textbf{Identifying Reliable Knowledge via Observation}.
Theorem~\ref{tm:class_threshold} gives $t_b$ a class-wise
reliability interpretation under calibration.
We therefore use $t_b$ as a class-specific baseline and construct
an observation matrix
$\mathbf{O}_{\mathring{y},y^{\ast}}$.
An entry $\mathbf{O}_{a,b}$ counts target-training instances that
are predicted as class $a$, have ground-truth class $b$, and whose
confidence in the true class exceeds the reliability baseline
$t_b$:
\begin{align}\label{eq:O_ab}
\mathbf{O}_{a,b}
=
\left|
\left\{
\mathbf{x}\in\mathbf{X}_{\mathring{y}=a}
\ \middle|\
y_{\mathbf{x}}^{\ast}=b,\,
\hat{p}_{b}(\mathbf{x})\geq t_b
\right\}
\right|.
\end{align}

3. \textbf{Deriving a Class Invariance Score}.
From this, we compute a normalized estimation matrix $\mathbf{E}_{\mathring{y},y^{\ast}}$ in Equation~\ref{eq:E_ab}.
Finally, we set $ \kappa_i = \mathbf{E}_{\mathring{y}=y_i,y^{\ast}=y_i } 
 $ with the diagonal element for each sample in the $\mathcal{L}_{\text{IBS}}$ loss.
\begin{align}\label{eq:E_ab}
\mathbf{E}_{a,b} = \frac{ \frac{ \mathbf{O}_{a,b}}{ \sum_{k=1}^{m} \mathbf{O}_{a,k}} \cdotp | \mathbf{X}_{\mathring{y}=a} |  }{  \sum_{j=1}^{m} \frac{ \mathbf{O}_{j,b}}{ \sum_{k=1}^{m} \mathbf{O}_{j,k}} \cdotp | \mathbf{X}_{\mathring{y}=j} | }
\end{align}

\subsection{Holistic Optimization Objective and Mutual Information Estimation}
The overall training objective combines task loss, IB losses, reconstruction loss, and distillation loss:
\begin{align}\label{eq:L_Total}
\mathcal{L}_{\mathrm{Total}} = \mathcal{L}_{\mathrm{Task}} + \lambda_{\mathrm{IB}}(\mathcal{L}_{\mathrm{IBT}} + \mathcal{L}_{\mathrm{IBS}}) + \lambda_{\mathrm{R}}\mathcal{L}_{\mathrm{R}} + \lambda_{\mathrm{KD}}\mathcal{L}_{\mathrm{KD}}
\end{align}
For mutual information terms, we use variational estimators: the Barber-Agakov lower bound for maximizing $I(\mathbf{Z}_{\phi}; \mathbf{Y})$~\cite{NIPS2003_a6ea8471} and $I(\mathbf{z}_{vs}; \mathbf{Y})$, and the CLUB upper bound~\cite{cheng2020club} for minimizing $I(\mathbf{z}_{vr}; \mathbf{Y})$, $I(\mathbf{z}_{vs}; \mathbf{z}_{vr})$, and $I(\mathbf{Z}_{\phi}; \Phi | \mathbf{Y})$.

\section{Experiments}

\subsection{Experimental Settings}
\textbf{Datasets.} 
 We evaluate DIB-OD on seven graph-classification benchmarks: NCI109, FRANKENSTEIN,
Mutagenicity, PROTEINS, D\&D, IMDB-BINARY, and
deezer\_ego\_nets. 
These datasets span molecular, bioinformatics, and social-network domains. Detailed statistics are provided in Appendix B. To reconcile cross-domain feature dimensions, we zero-pad target-domain node features to a common dimension before fine-tuning.

\textbf{Baseline Models.} 
We compare against GAT~\cite{velivckovic2017graph}, DANN~\cite{ganin2016domain}, DANE~\cite{zhang2019dane}, ASN~\cite{zhang2021adversarial}, UDAGCN~\cite{wu2020unsupervised}, GRADE~\cite{wu2023non}, JHGDA~\cite{shi2023improving}, and  covering graph representation learning and representative adversarial, structure-aware, and hierarchical graph-transfer approaches. 
For more baseline details, please refer Appendix~C.
%

%
%

\textbf{Default settings.}
Accuracy is the primary metric, and both pre-training and
target-domain fine-tuning use 10-fold cross-validation. 
The backbone comprises a three-layer GCN followed by an MLP
classifier. 
The teacher and student are jointly optimized during pre-training, whereas the teacher is frozen during fine-tuning. Both stages run for 100 epochs using Adam with a learning rate of $10^{-3}$. 
We set
$\beta_T=0.1$,
$\beta_y=\beta_{vs}=\lambda_{\mathrm{Orth}}=0.05$,
$\lambda_{\mathrm{IB}}=\lambda_{\mathrm{KD}}=0.01$, and
$\lambda_R=0.001$.

\subsection{Main Result}

\textbf{Analyzing Variation in Data Performance Among Comparable Collections.} We conduct experiments to test DIB-OD’s performance in single cross-domain learning, as shown in Table~\ref{table:DIB-OD_Intra_Type_Domain}.
To validate our model's capability for knowledge transfer, we design a series of cross-domain pre-training experiments within each category of comparable datasets. 
Specifically, we conduct experiments in the three distinct domains outlined previously.
For clarity, we denote the direction of cross-domain transfer using an arrow ($\rightarrow$) between the abbreviated names of the datasets. 
%
%
This notation is consistently applied throughout our experiments to indicate the source and target domains.


\textbf{Results Analysis.}
As shown in Table~\ref{table:DIB-OD_Intra_Type_Domain}, DIB-OD achieves the best mean accuracy in all eight within-type transfer settings. In particular, DIB-OD outperforms the strongest competing baselines, DANE and UDAGCN, by approximately $11.1$ and $13.1$ percentage points on M$\rightarrow$N and P$\rightarrow$D, respectively. These gains support the effectiveness of the proposed information-allocation and preservation mechanisms under transfers between related chemical and biological datasets.
 DIB-OD also achieves the best mean accuracy on both social-network transfers, I$\rightarrow$d and d$\rightarrow$I, despite the generally lower performance in this domain.
 %
 %
 By optimizing the aggregation of "decisive information" from diverse structural views, DIB-OD effectively constructs a more complete invariant core that captures fundamental protein motifs, which remain consistent despite variation in protein sizes.
(2) Stability in social networks: For social network transfers (I$\rightarrow$d and d$\rightarrow$I), while the overall performance of all models is relatively lower due to the high sparsity of ego-networks, DIB-OD still maintains a steady lead. 

\begin{table*}[t]
\caption{Intra-Domain Generalization Performance}
\tabcolsep 2 pt
\label{table:DIB-OD_Intra_Type_Domain}
\centering
\renewcommand{\arraystretch}{1.2}
\begin{tabular}{c| c c c c c c| c c | c c }

\toprule[1pt]
Datasets & $N \rightarrow F$ &  & $M \rightarrow  N$  &  $F \rightarrow  M$  &   $M \rightarrow  F$  & $ P \rightarrow D$ &   $ D \rightarrow P$    &   $ I \rightarrow d$   &   $ d  \rightarrow I$       \\ 
\midrule[0.5pt]

GAT & 44.67 ± 0.12 &    & 49.71 ± 0.57  &  55.46 ± 0.13 &  53.63± 0.55   &   50.73 ± 1.34    &  60.58 ± 0.71   &   43.39 ± 0.05   &   48.94 ± 1.16 \\\hline
DANN &  45.40 ± 1.36   &    & 49.45 ± 0.34 &  55.61 ± 0.46   & 48.73 ± 3.21   & 43.41 ± 4.74  &  50.37 ± 1.09  & 46.18 ± 1.50  &   50.92 ± 1.84  \\
UDAGCN & 56.57 ± 2.80 &    & 48.62 ± 6.98 &  50.61 ± 3.53  &  48.35 ± 5.45   &  60.68 ± 0.15  &  49.47 ± 1.14   & 49.47 ± 1.14   & 50.76 ± 0.11 \\
DANE & 51.24 ± 2.70  &     & 57.41 ± 0.99 & 52.71 ± 5.73   &     56.61 ± 2.69   &  48.72 ± 2.64 &    61.84 ± 2.10 &  48.72 ± 2.64 &  46.30 ± 1.29   \\
ASN & 53.22 ± 4.29  &    & 49.62 ± 0.06  & 44.64 ± 0.01 & 44.64 ± 0.00  &   56.64 ± 0.00  &  58.66 ± 0.00 & 43.19 ± 0.00 &   50.01 ± 0.02   \\
GRADE & 53.22 ± 4.29  &     &  50.80 ± 1.76  & 51.07 ± 5.25  &  46.80 ± 4.30     &  44.33 ± 3.11  &  42.21 ± 2.38 & 47.50 ± 0.63  & 50.86 ± 0.00 \\\hline
JHGDA  & 51.41 ± 4.98&    &  51.21 ± 0.91&  55.40 ± 0.09  &   48.56 ± 5.00  &  59.55 ± 0.45 &  58.44 ± 4.25 &  50.43 ± 4.50 &  51.36 ± 0.22   \\
DIB-OD & \textbf{57.30 ± 3.04}  &            &  \textbf{ 68.50 ± 2.67 }  &  \textbf{ 56.21 ± 2.68}  &  \textbf{57.09 ± 2.32}  &  \textbf{73.77 ± 2.90} &  \textbf{68.90 ± 5.72} & \textbf{54.04 ± 2.65}  &  \textbf{53.43 ± 2.02 }\\
\bottomrule[1pt]
\end{tabular}
\end{table*}

\subsection{Generalization Across Heterogeneous Domains
}

To rigorously evaluate the robustness and transferability of our pre-training framework, we further design a challenging set of experiments that involve knowledge transfer across heterogeneous domains with significant distribution shifts. 
Specifically, we investigate the scenario where a model is pre-trained on a dataset from one broad category  and subsequently fine-tuned and evaluated on datasets from a structurally and semantically different category.

Table~\ref{table:DIB-OD_Extra_Type_Domain} presents the results for inter-type transfers. The performance gap between DIB-OD and existing methods is even more pronounced here than in intra-domain tasks.
(1) Overcoming the "Generalization Wall": Baselines like UDAGCN and ASN, which perform reasonably well in intra-domain tasks, suffer from significant performance degradation in cross-type scenarios (e.g., 
N $\rightarrow$ P
 or M $\rightarrow$ D). 
In contrast, DIB-OD maintains high accuracy, achieving a 70.17\% accuracy in N $\rightarrow$ P, outperforming GAT by over 10\%. This proves that our framework successfully disentangles universal structural laws from category-specific styles (e.g., chemical valency vs. biological folding patterns).
(2) Effectiveness of Decoupled Representations: The most striking improvement is observed in M $\rightarrow$ D
 and I $\rightarrow$ P  transfers, where DIB-OD exceeds all baselines by more than 10-15 percentage points. 
 This suggests that by explicitly separating redundant information via HSIC and the Decoupled IB, DIB-OD prevents ``negative transfer", where source-specific noise (e.g., IMDB's specific social connectivity) would otherwise contaminate the target domain's learning process.
(3) Consistency Across Directions: Whether transferring N $\rightarrow$ P or P $\rightarrow$ N, DIB-OD provides symmetrically robust results (70\% and 69\% respectively). 
This symmetry indicates that the identified invariant core is truly representative of a universal graph language that transcends specific domain boundaries. 
\begin{table*}[!t]
\setlength{\tabcolsep}{1mm}
\caption{\textbf{Cross-Type Data Performance Comparison}}
\label{table:DIB-OD_Extra_Type_Domain}
\centering
\renewcommand{\arraystretch}{1.2}
\begin{tabular}{c| c c c c c c   c  c c }

\toprule[1pt]
Datasets & $N \rightarrow P$ & $ P  \rightarrow N$ & $M \rightarrow P $ &  $P \rightarrow  M $  &   $M \rightarrow  D$  & $ D \rightarrow M $    &       $ I  \rightarrow P$   &  $D \rightarrow  N$\\ 

GAT & 59.57 ± 0.01& 50.38 ± 0.01 & 52.56 ± 1.18 & 55.35 ± 0.23 & 49.19 ± 1.20 & 44.78 ± 0.18  & 40.42 ± 0.03  & 49.38 ± 0.38  \\\hline

DANN &  55.58 ± 8.72 & 50.38 ± 0.00 & 61.40 ± 2.23 & 44.84 ± 0.41 & 44.64 ± 0.00 & 44.60 ± 0.07   & 40.38 ± 0.07   &  50.87 ± 0.90  \\
UDAGCN & 44.87 ± 0.25 & 45.56 ± 4.59  &  54.16 ± 0.83 &   \textbf{73.29  ± 0.00} &  51.15 ± 3.11 & 52.20 ± 2.88  &  49.98 ± 0.23 & 43.30 ± 0.95\\

DANE & 34.91 ± 3.44 &  41.89 ± 2.29 & 52.59 ± 2.08 & 59.75 ± 1.13 &  47.26 ± 7.32 &   46.11 ± 5.16  & 60.62 ± 3.88  &    47.63 ± 3.45\\
ASN &  40.43 ± 0.00 &  49.62 ± 0.01 & 59.57 ± 0.00 &  55.34 ± 0.02  & 58.66 ± 0.00   & 55.36 ± 0.00   &  59.75 ± 0.01    &  49.62 ± 0.03\\

GRADE & 59.48 ± 0.08   & 50.51 ± 0.69  & 49.31 ± 8.60  & 48.94 ± 5.27   &  47.79 ± 8.91  &  58.53 ± 7.28  & 40.49 ± 0.49   & 50.08 ± 0.37 \\\hline

JHGDA  & 57.99 ± 4.16 & 53.11 ± 0.54 & 59.59 ± 0.04&  58.11 ± 0.66  & 59.13 ± 1.06 &  54.05 ± 3.89 & 50.68 ± 7.19     &   53.67 ± 0.93\\
DIB-OD & \textbf{70.17 ± 2.41}  & \textbf{69.28 ± 2.71} & \textbf{69.81 ± 3.28}   &  70.28 ± 3.21  & \textbf{74.53 ± 3.97} & \textbf{69.63 ± 3.52}   &    \textbf{70.80 ± 3.33}       & \textbf{67.82 ± 2.63} \\
\bottomrule[1pt]
\end{tabular}
\end{table*}

\subsection{ABLATION STUDIES}
\subsubsection{synthesizes and analyzes.}
%
We select four intra-type and four inter-type transfer pairs to cover different source-target directions and domain categories.
The results are summarized in Table~\ref{table:ablation_hemogenous_hetergenous}.

\begin{table*}[!ht]
\caption{Ablation Results (\%) on Within-type and Cross-type  Graph Transfer}
\tabcolsep 2 pt
\label{table:ablation_hemogenous_hetergenous}
\centering
\renewcommand{\arraystretch}{1.2}
\begin{tabular}{c| c c c c | c c   c  c c }
\toprule[1pt]
Datasets & $N \rightarrow F$ & $ F  \rightarrow N$ & $M \rightarrow N $ &  $P \rightarrow  D $ & $N \rightarrow P$ & $ P  \rightarrow N$ &  $M \rightarrow  D $  &  $ I  \rightarrow P$    \\ 
\midrule[0.5pt]
w/o  IB  & 56.75 ± 1.70 & 69.64 ± 2.56 & 51.22 ± 2.31  &  70.37 ± 8.19 &  59.57 ± 0.01& 68.04 ± 2.66 & 73.59 ± 5.20 & 69.73 ± 5.36      \\   
w/o  HSIC &  56.56 ± 2.72 &  69.30 ± 3.03  &  68.21 ± 2.64  & 72.82 ± 6.16 & 67.30 ± 5.31 & 67.70 ± 2.98 & 73.69 ± 5.38 & 67.66 ± 3.43   \\
w/o SSR  & 56.79 ± 2.60 & 52.10 ± 3.70  &  51.44 ± 5.72  &  73.51 ± 4.84 & 53.42 ± 2.25 & 68.06 ± 2.52  &  72.58 ± 4.53 &   68.46 ± 5.87  \\
Full Fine-tuning & 56.42 ± 3.37 &  52.56 ± 2.61 & 52.29 ± 3.64  &   73.34 ± 4.05  & 69.63 ± 3.65 &  67.75 ± 3.12 & 67.84 ± 2.26 & 67.94 ± 3.16  \\\hline
Ours &  \textbf{57.30 ± 3.04} &  \textbf{69.83 ± 2.91 }  & \textbf{68.50 ± 2.67} &  \textbf{73.77 ± 2.90} &  \textbf{70.17 ± 2.41} & \textbf{69.28 ± 2.71} &  \textbf{74.53 ± 3.97 }
  & \textbf{ 70.80 ± 3.33 }\\
\bottomrule[1pt]
\end{tabular}
\end{table*}

\begin{figure*}[!t]
    \centering

    \begin{minipage}[t]{0.31\textwidth}
        \centering
        \includegraphics[width=\linewidth]
        {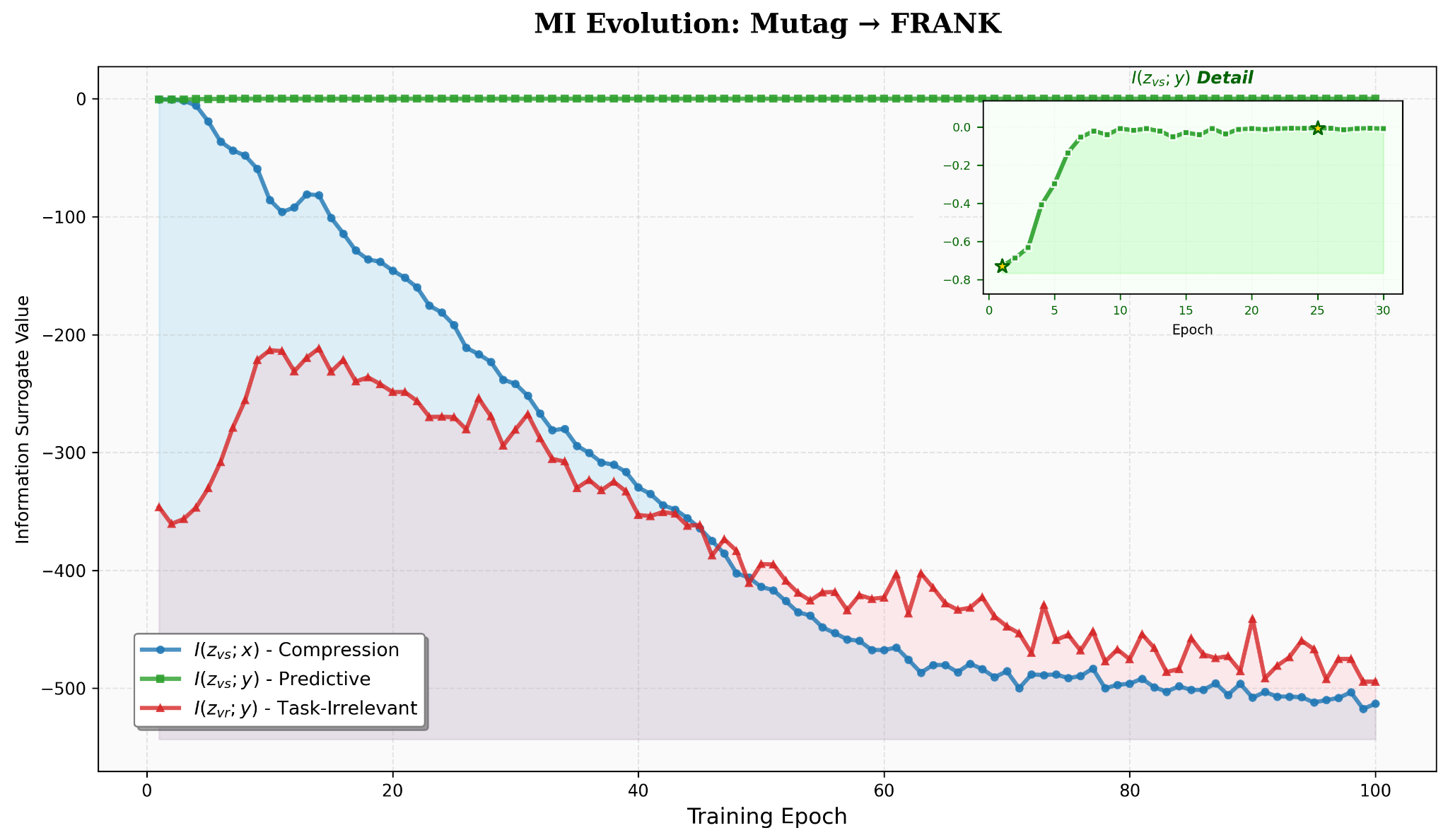}

        \small (a) Mutagenicity $\rightarrow$ FRANKENSTEIN
    \end{minipage}
    \hfill
    \begin{minipage}[t]{0.31\textwidth}
        \centering
        \includegraphics[width=\linewidth]
        {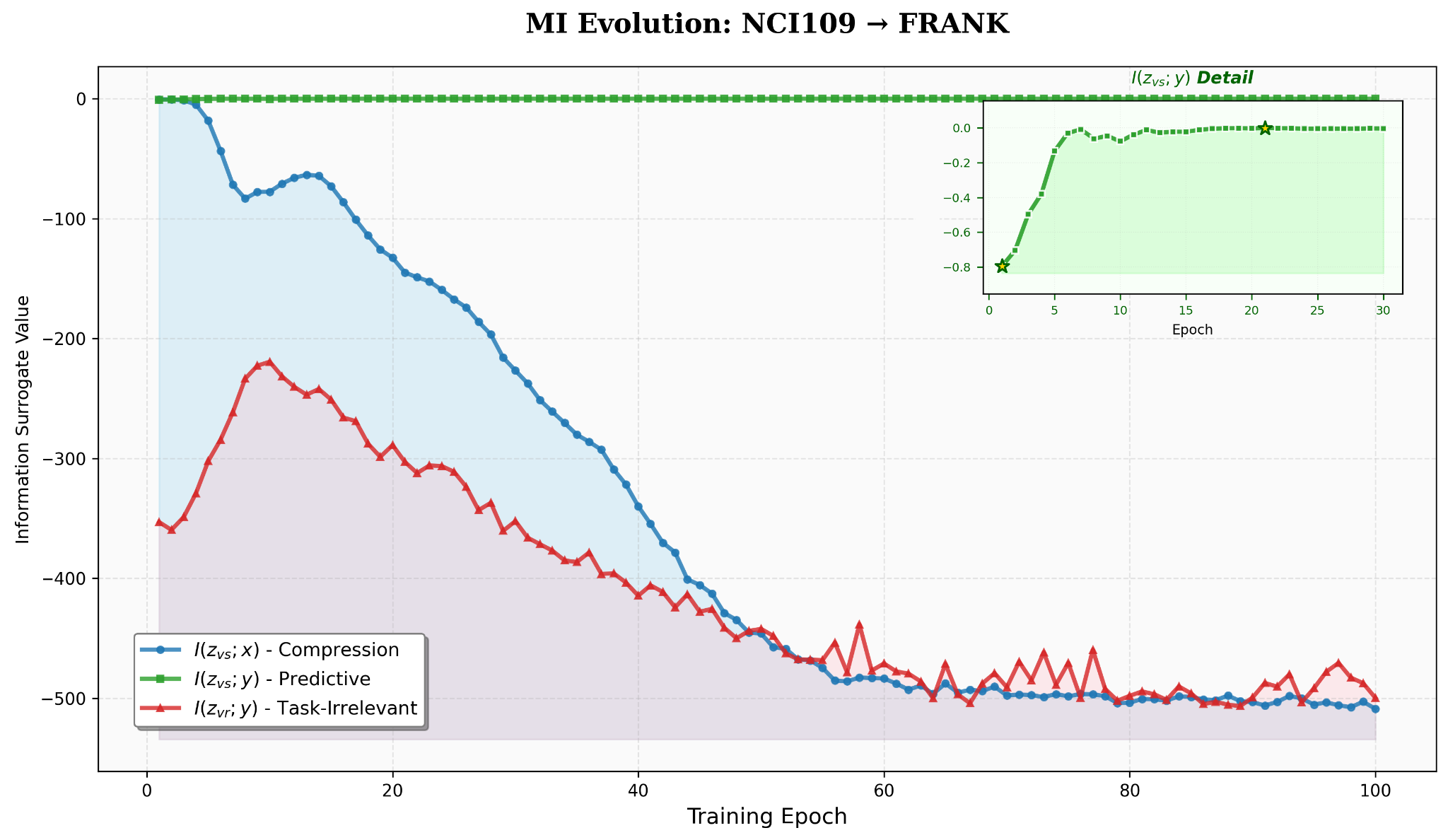}

        \small (b) NCI109 $\rightarrow$ FRANKENSTEIN
    \end{minipage}
    \hfill
    \begin{minipage}[t]{0.31\textwidth}
        \centering
        \includegraphics[width=\linewidth]
        {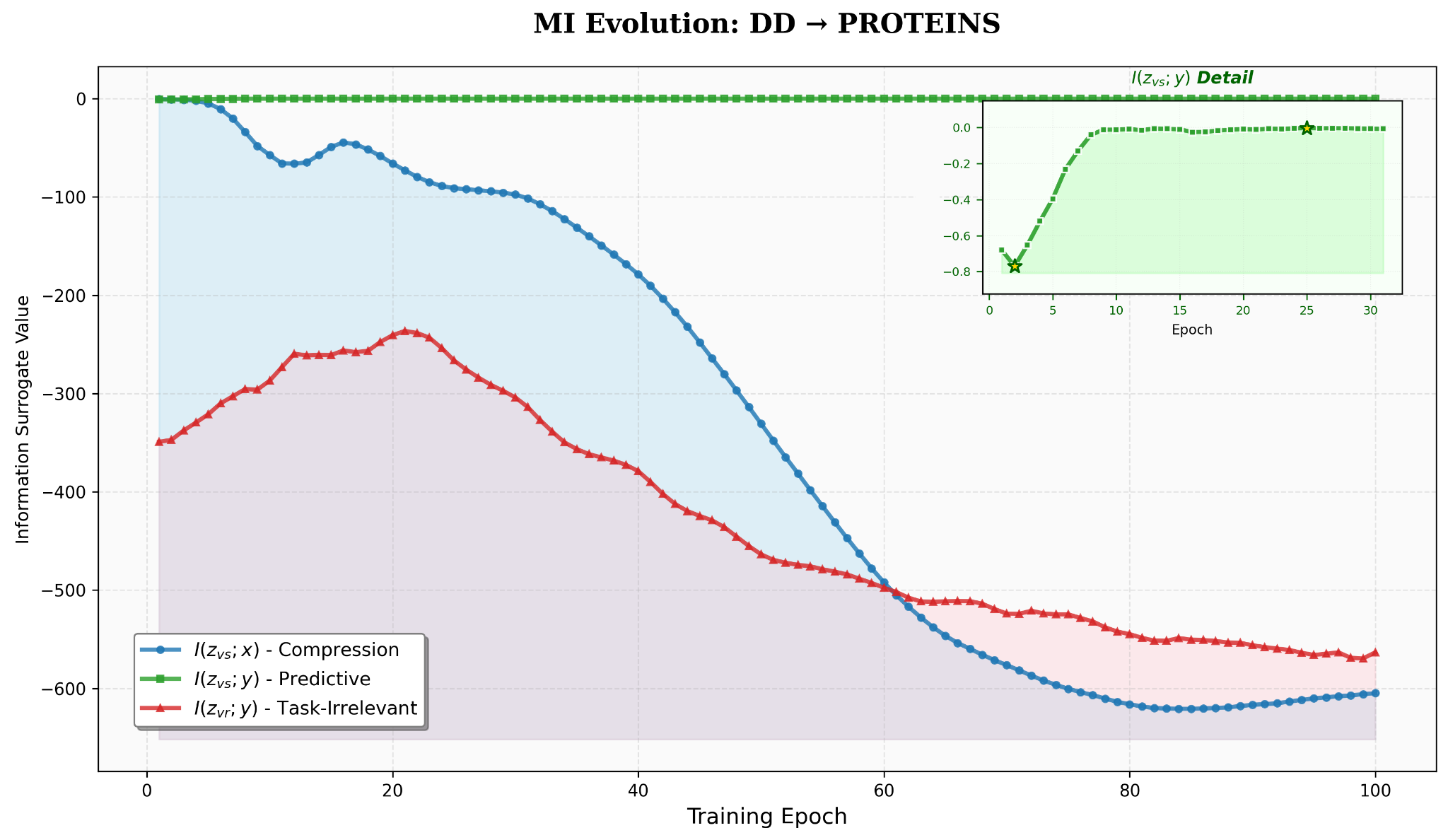}

        \small (c) DD $\rightarrow$ PROTEINS
    \end{minipage}

    \par\medskip

    \begin{minipage}[t]{0.31\textwidth}
        \centering
        \includegraphics[width=\linewidth]
        {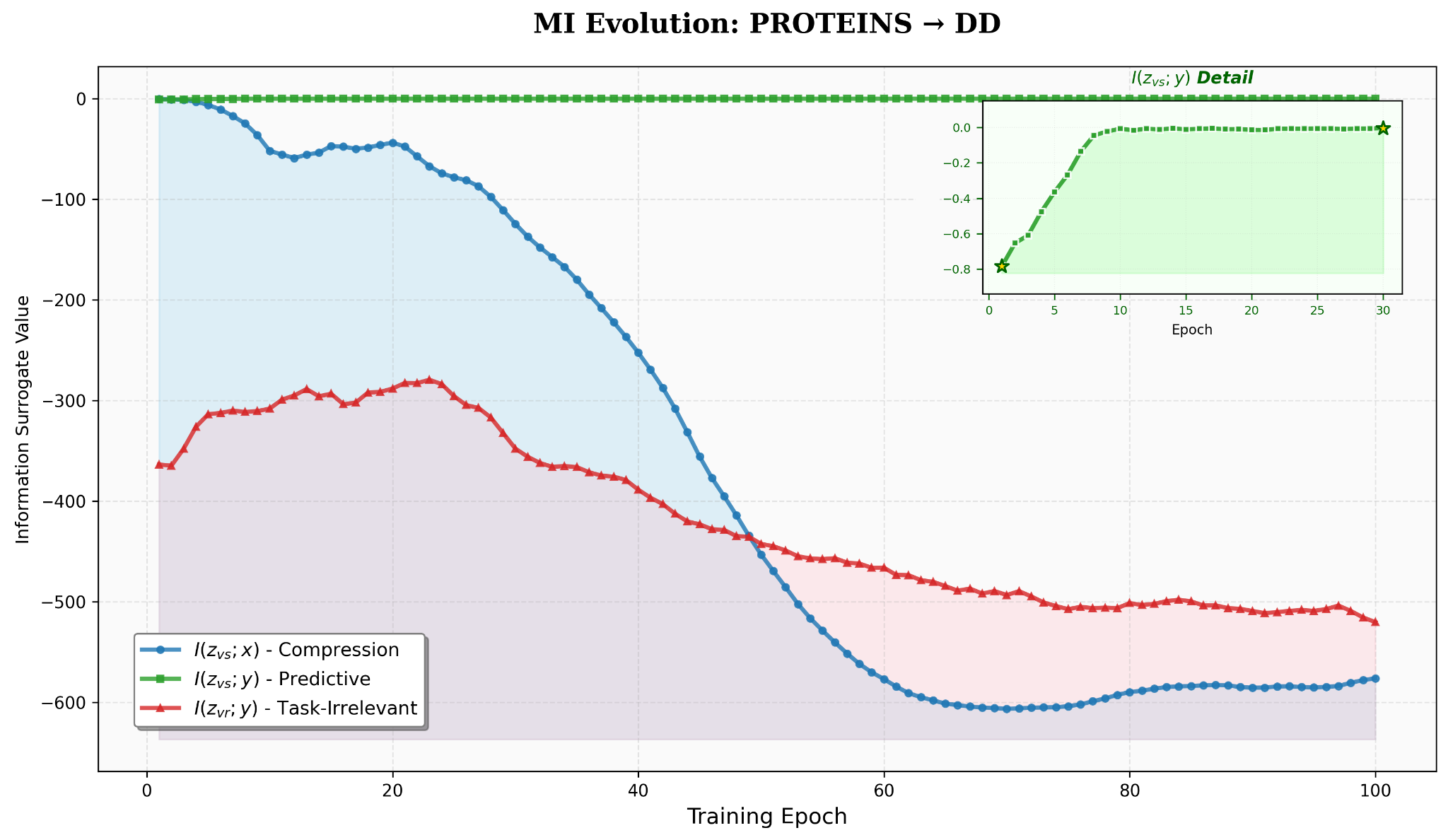}

        \small (d) PROTEINS $\rightarrow$ DD
    \end{minipage}
    \hfill
    \begin{minipage}[t]{0.31\textwidth}
        \centering
        \includegraphics[width=\linewidth]
        {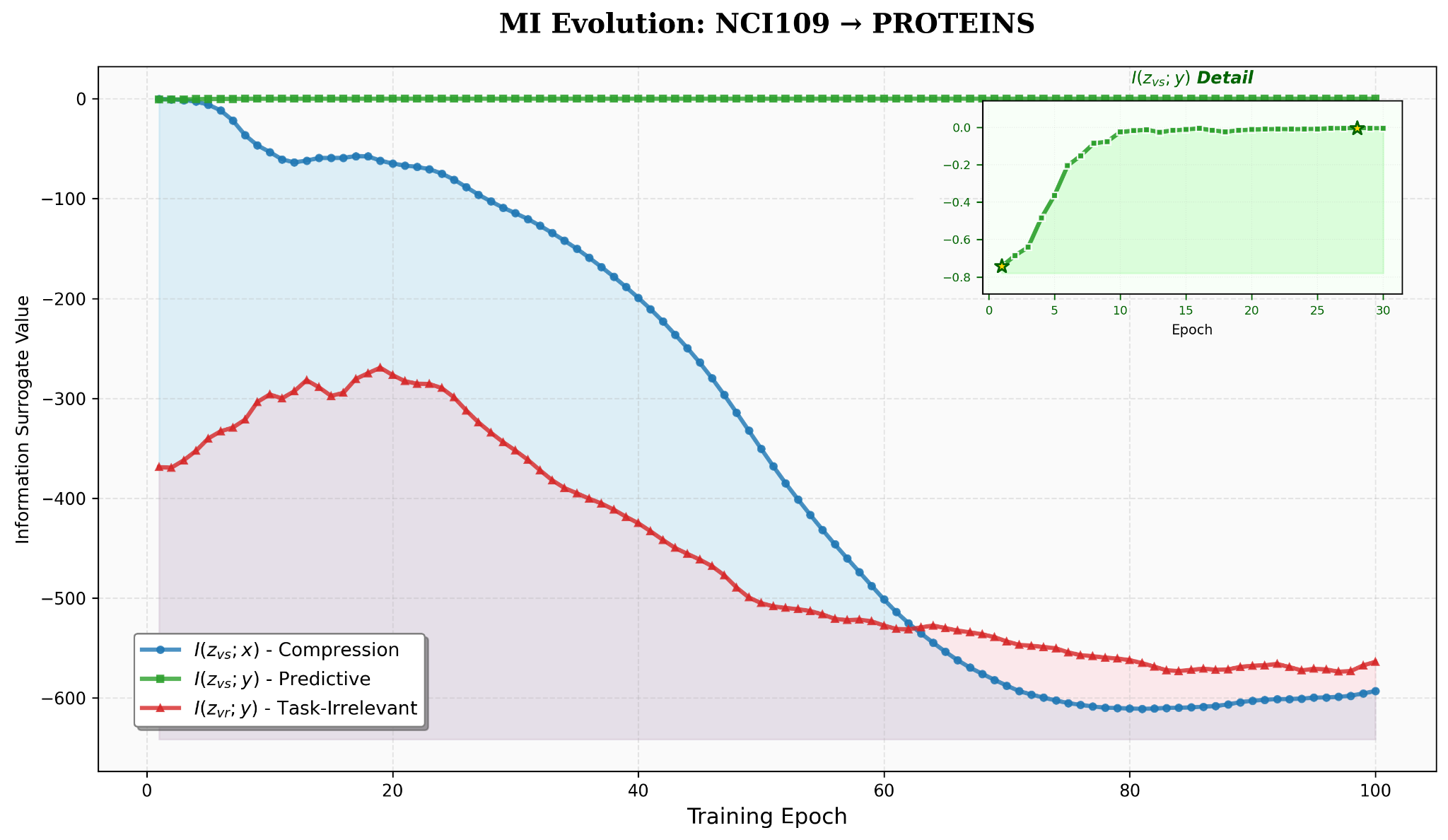}

        \small (e) NCI109 $\rightarrow$ PROTEINS
    \end{minipage}
    \hfill
    \begin{minipage}[t]{0.31\textwidth}
        \centering
        \includegraphics[width=\linewidth]
        {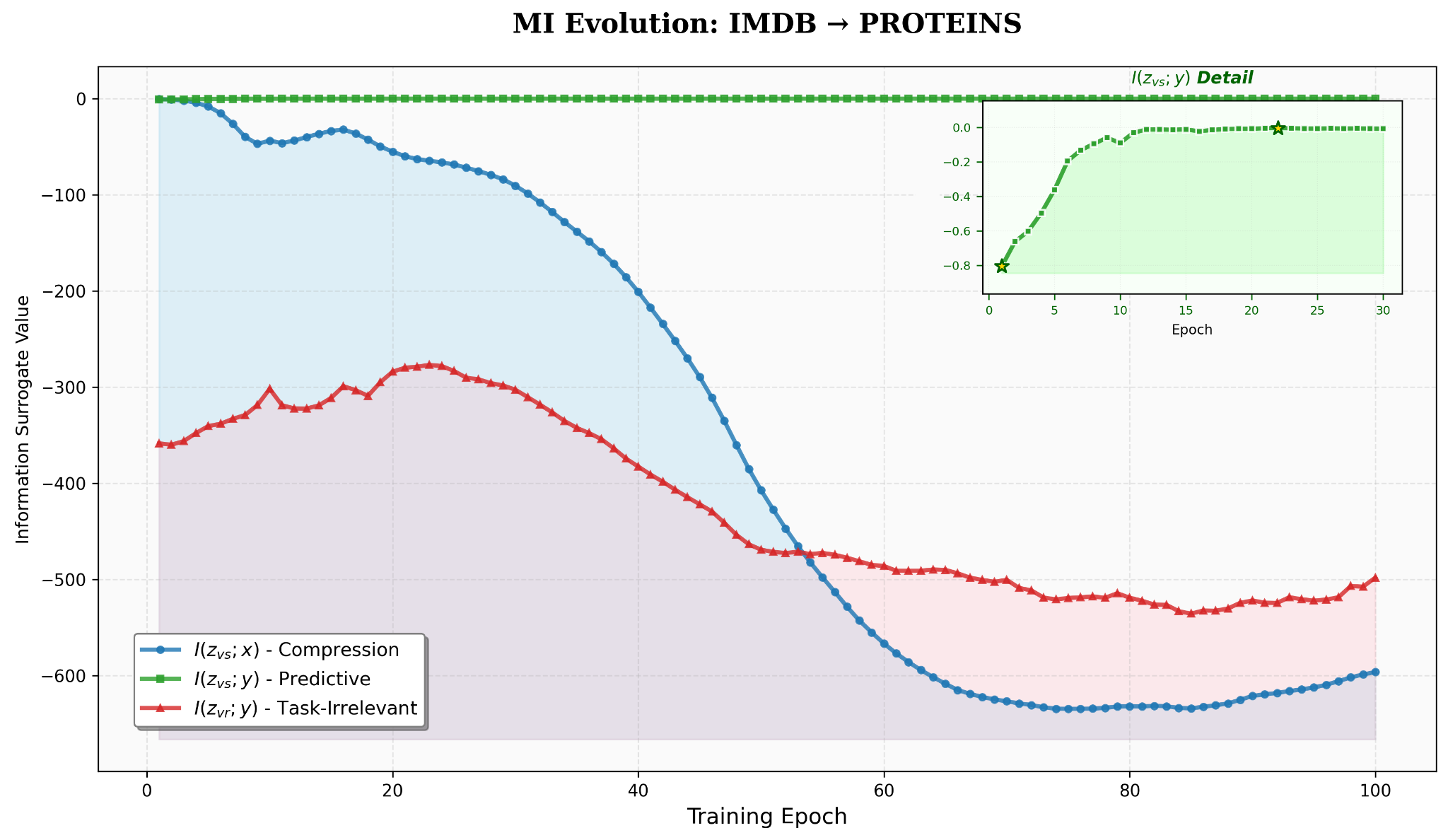}

        \small (f) IMDB-BINARY $\rightarrow$ PROTEINS
    \end{minipage}

    \caption{Changes in mutual information during cross-domain adaptation.}
    \label{fig:totalMIstate}
\end{figure*}

\textbf{Efficacy of SSR:}~The removal of the SSR (setting $\kappa = 1$) leads to a significant performance decline, particularly in challenging transfers like F $\rightarrow$ N (from 69.83\% to 52.10\%) and M $\rightarrow$ N (from 68.50\% to 51.44\%). 
This drop suggests that without class-aware adaptive weighting, the model treats all structural knowledge equally, allowing noisy target-specific signals to overwrite the hard-won "invariant core."
The SSR's ability to protect high-confidence semantic knowledge is crucial for maintaining the stability of representations during adaptation.

\textbf{Invariant-Core Preservation under Target-Domain Adaptation:}~The comparison with Full Fine-tuning where the teacher model is not frozen highlights the risk of knowledge corruption. 
In scenarios such as N $\rightarrow$ F
, full fine-tuning results in lower accuracy (56.42\%) compared to DIB-OD (57.30\%). 
This confirms that keeping a stable "frozen teacher" as a knowledge anchor prevents the model from over-adjusting to the target domain's distribution, thereby preserving the generalizable graph principles captured during pre-training.

\textbf{The Necessity of Information Compression:}~
%
Removing the IB terms causes a substantial degradation on several tasks. The drop is particularly pronounced on the intra-type M~$\rightarrow$~N transfer, from 68.50~\% to 51.22~\%, and on the inter-type N~$\rightarrow$~P transfer, from 70.17~\% to 59.57~\%.
The compression terms $I(\mathbf{Z}_{\phi}; \mathbf{X}_{\Phi})$ and $ I(\mathbf{z}_{vs}; \mathbf{z}_{vr})$ serve as filters that discard domain-specific "redundant styles. 
Without this bottleneck, the student model's $\mathbf{z}_{vs}$ becomes contaminated with spurious correlations from the source domain, leading to severe negative transfer when the model encounters the different connectivity logic of the target domain.

\textbf{The Role of Orthogonal Decoupling:} The results for removing HSIC further validate our decoupling strategy. 
%
%
Removing HSIC decreases accuracy by 2.87~\% and 1.58~\% on N$\rightarrow$P and P$\rightarrow$N, respectively.
This empirical evidence supports our theoretical claim: while IB guides the information flow, the HSIC-based regularizer ensures that the invariant and redundant subspaces are statistically independent.
Without this forced separation, the invariant core fails to isolate itself from category-dependent noise, which limits the model’s ability to bridge fundamentally different graph categories.
\subsubsection{Information Bottleneck Dynamics}
%





\subsection{Ablation Studies}

Table~\ref{table:ablation_hemogenous_hetergenous} evaluates the contribution of the IB,
HSIC, SSR, and frozen-teacher components on four within-type and
four cross-type transfer directions.
Removing SSR causes the largest drops on
F$\rightarrow$N, M$\rightarrow$N, and N$\rightarrow$P,
showing the importance of class-aware preservation during target
adaptation.
Removing the IB terms also substantially degrades
M$\rightarrow$N and N$\rightarrow$P, supporting the role of
information compression and branch allocation.
Removing HSIC produces smaller but consistent reductions across all
eight directions, with the largest drops on
N$\rightarrow$P and I$\rightarrow$P.
Finally, unfreezing the teacher is particularly harmful on
F$\rightarrow$N, M$\rightarrow$N, and M$\rightarrow$D,
which is consistent with the frozen teacher serving as a stable
adaptation anchor.

\section{Conclusion}

We propose DIB-OD for robust heterogeneous graph adaptation, disentangling a task-relevant invariant core from domain-specific redundancy via a decoupled information bottleneck and online distillation—where the core preserves predictive information and the residual captures instance-level details.
 A confidence-aware semantic regularizer, dynamically tuned by class-wise reliability, further protects transferable knowledge during target adaptation. Extensive experiments on chemical, biological, and social graphs demonstrate consistent improvements, with substantial gains under challenging cross-type shifts, confirming the efficacy of our decomposition strategy.


\appendix


\bibliography{aaai2027}
\clearpage

\bigskip




\end{document}